\documentclass[twoside,11pt]{article}
\usepackage{jmlr2e}

\usepackage{amssymb}
\usepackage{latexsym}
\usepackage{subcaption}
\usepackage{amsfonts}
\usepackage{graphicx} 
\usepackage{mathtools}
\usepackage{amsmath}
\usepackage[ruled,vlined]{algorithm2e}
\DeclarePairedDelimiter\ceil{\lceil}{\rceil}

\newcommand{\X}{\mathcal{X}}
\newcommand{\R}{\mathbb{R}}
\newcommand{\N}{\mathbb{N}}
\newcommand{\eps}{\epsilon}
\newcommand{\err}{\operatorname{err}}
\newcommand{\tips}{\tilde\eps}

\newcommand{\ds}{\displaystyle}
\newcommand{\set}[1]{\left\{ #1 \right\}}
\newcommand{\oo}[1]{\frac{1}{#1}}
\newcommand{\paren}[1]{\left( #1 \right)}

\DeclareMathOperator{\ddim}{ddim}
\DeclareMathOperator{\dens}{dens}
\DeclareMathOperator{\diam}{diam}
\DeclareMathOperator{\inn}{in}
\DeclareMathOperator{\out}{out}
\DeclareMathOperator{\prop}{prop}
\DeclareMathOperator{\OPT}{OPT}

\graphicspath{
    {Figures}
}

\firstpageno{1}

\begin{document}

\title{Classification in asymmetric spaces via sample compression}

\author{\name Lee-Ad Gottlieb \email leead@ariel.ac.il\\
\name Shira Ozeri \email shirahalevy2@gmail.com\\
\addr Department of Computer Science\\
Ariel University\\
Ariel, Israel}

\editor{}

\maketitle

\begin{abstract}
We initiate the rigorous study of classification
in quasi-metric spaces. These are point sets endowed
with a distance function that is non-negative
and also satisfies the triangle inequality, 
but is asymmetric.
We develop and refine a learning algorithm for quasi-metrics
based on sample compression and nearest neighbor,
and prove that it has favorable statistical properties.
\end{abstract}


We initiate the rigorous study of classification in {\em quasi-metrics}.
These are spaces endowed with a distance function that is non-negative,
obeys the triangle inequality, but not symmetric.
The term `quasi-metric' appears as early as \cite{W-31}, and it has
been the subject of significant research in such areas as topology \citep{K-01}
and theoretical computer science.
As pointed out by \cite{L-73},
quasi-metrics occur naturally in many settings and applications, 
such as directed graphs and Hausdorff distances on certain subsets 
of metric space.\footnote{We not that the Hausdorff distance may not obey the
triangle inequality.}
More simply, travel times on road networks are quasi-metrics (due for example
to traffic and one-way streets), as are travel times on uneven terrain, since
marching up to the top of a hill takes more time than marching down again.
For this reason, there has been significant work addressing
the Travelling Salesman Problem in asymmetric spaces
\citep{FGM-82, AGMGS-10, AG-15, STV-17}.

Turing to classification, if we wish to classify in quasi-metric spaces --
for example, determine whether an unknown village belongs to one country
or another, based on its proximity to known villages -- 
we require classification tools that are resilient to asymmetry.
In general, we inhabit an inherently asymmetric world, yet we are unaware
of any rigorous study of learning in quasi-metric spaces. 
Indeed, much of the existing machinery for classification algorithms, as well as
generalization bounds, depend strongly on the axioms of the metric spaces, and so
do not immediately transfer over to quasi-metric space.
Our goal in this paper is to introduce techniques, tools, and statistical
analysis for learning in this setting, for which no classification
guarantees were previously known.

Our task is aided by a preexisting framework for learning in metric spaces
of low intrinsic dimensionality.
\citep{luxburg2004distance,gottlieb2014efficient,kontorovich2014maximum}.
In these spaces, it is known that a small sample is sufficient to achieve
classification with low generalization error via the nearest neighbor classifier.
It is also known that dependence on the dimensionality is unavoidable 
\citep{shalev2014understanding}.
This framework proved sufficiently powerful to extend to non-metric space
such as semi-metrics (which do not obey the triangle inequality), although
this extension required developing a new definition of dimensionality
\citep{gottlieb2017nearly}. We wish to use this framework as a foundation
for learning in quasi-metrics as well, but the weak structure of quasi-metric
make this a non-trivial task.

\paragraph{Our contribution.}
We present a rigorous approach to learning in quasi-metric spaces.
We define a new measure of dimensionality for quasi-metric spaces,
and show how this measure can be used for sample compression
(Section \ref{sec:cover}).
We then present a classifier based on compression and proximity, 
and prove strong generalization bounds for it 
(Section \ref{sec:gen}).
Our classification framework implies a range of new algorithmic 
questions which we address in Section \ref{sec:alg}. 
There we explore different approaches
to sample compression for quasi-metrics, as well as prove
the complexity of evaluation time for our classifier.

Finally, we turn to some simple metrization techniques for quasi-metrics,
and show that while these techniques can transform the quasi-metric
into metric or semi-metric spaces, the transformation typically induces
a degradation in some property necessary for learning
(dimension or margin), rendering this approach undersirable
(Section \ref{sec:trans}).

\paragraph{Related work.}
As mentioned, quasi-metric were a subject of mathematical study as early as the
1930's \citep{W-31}.
Very early approaches to these spaces already attempted `metrization'
to transform them into the more malleable metric spaces \citep{F-37},
but more recently the very limited nature of this approach has been
acknowledged \citep{S-06, DTH-19}.
Other properties of quasi-metrics have been studied as well: For example,
\cite{S-69} considered the relationship between quasi-metrics and Moore spaces, 
with emphasis compactness and metrizability.
(See also follow up work by \cite{R-76}.)
\cite{D-88} studied a notion of Cauchy sequences in quasi-metrics, while
\cite{G-17} introduced and studied Lipschitz-regular quasi-metric spaces.
Recently, \cite{MSS-18} studied generalizations of classical metric 
embeddings to the quasimetric setting,
for example embedding quasi-metric spaces into ultra-quasi-metrics.
and these have many algorithmic applications.
We refer the reader to these papers for many additional references.

Quasi-metrics have also appeared in a number of machine learning
applications. For example,
\cite{GAB-02} introduced a quasi-metric operator,
while others focused on computable analysis \citep{CTM-05}  
or optimization \citep{CTM-06}.
The performance of nearest neighbor search under quasi-metric
distance has also been studied \citep{KSTSH-18, ZGCCP-19},
and \cite{S-08} showed that similarity between peptides can be
modeled via quasi-metrics.
As previously stated, none of these present a 
rigorous classification framework for quasi-metrics.

\section{Preliminaries}\label{sec:pre}

\paragraph{Notation and basic concepts.}
We describe the recursive logarithm as $\log^{(1)}n = \log n$
and $\log^{(x)}n = \log \log^{(x-1)}n$ for $x>1$.
For example, 
$\log^{(2)}n = \log \log n$
and
so $\log^{(3)}n = \log \log \log n$,
etc.
The iterative logarithm $\log^* n$
is the smallest integer $i$ satisfying
$\log^{(i)}n \le 1 $.

A $k$-{\em hierachically well-separated tree} ($k$-HST) \citep{Bartal-96}
has the property that in any root-to-leaf path in the tree, 
the edge lengths decrease by a factor of exactly $k$ in each step.
See figure \ref{fig:k-hst}. 

\begin{figure}[ht]
    \centering
    \includegraphics[width=50mm,scale=0.5]{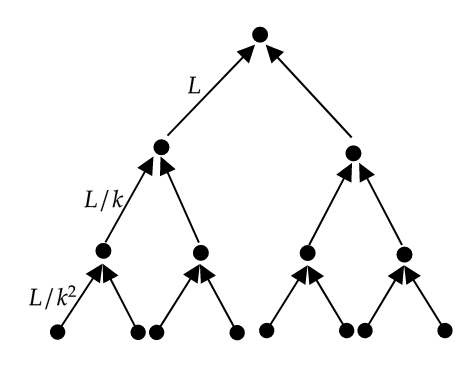}
    \caption{A $k$-HST with height 3.}
    \label{fig:k-hst}
\end{figure}

\paragraph{Distance spaces.} 
A distance function 
$\rho : \mathcal{X} \times \mathcal{X} \rightarrow \R^1$
defines the distance between two points of the set $\X$. 
For two sets $A,B$, we define
$\rho(A,B) = \min_{a \in A, b \in B} \rho(a,b)$.
Likewise,
$\rho(a,B) = \min_{b \in B} \rho(a,b)$.

A {\em metric} space $(\mathcal{X}, \rho)$
is an instance space $\mathcal{X}$ endowed with a distance function $\rho$
that is non-negative, symmetric 
($\rho(x,y) = \rho(y,x)	\;	\forall x,y \in \mathcal{X}$) 
and obeys the triangle inequality: 
$\rho(x,y) \le \rho(x,z) + \rho(z,y)	\;	\forall x,y,z \in \mathcal{X}$.
Often, one requires also that the distance function satisfy $\rho(x,y)=0 \Leftrightarrow x=y$.

In a {\em semi-metric} space $(\mathcal{X}, \rho)$, 
the distance function $\rho$ obeys the above metric conditions, 
with the exception of the triangle inequality.
In {\em quasi-metrics}, the distance function obeys the above
metric conditions, with the exception of only the symmetry
property.

In all cases, the {\em diameter} of $\X$ is defined as
$\diam(\X) = \max_{x,y \in \X} \rho(x,y)$.

\paragraph{Balls and dimension.} 
For a metric or semi-metric space $(\X,\rho)$, 
define ball $B_r(x) \subset \mathcal{X}$
to be all points of $\mathcal{X}$ within distance $r$ of some point $x$.
Let $\lambda = \lambda (\mathcal{X},\rho)$ be the smallest value such that for every 
radius $r$ and center-point $x \in \mathcal{X}$,
$B_r(x)$ can be covered by $\lambda$ balls of radius $\frac{r}{2}$. 
Then $\lambda$ is the {\em doubling constant} of $(\mathcal{X},\rho)$. 
The {\em doubling dimension} of $\mathcal{X}$ is defined 
as $\ddim(\X,\rho)=\log_2\lambda(\X,\rho)$ \citep{Assouad83, GKL03}.

For a metric or semi-metric space $(\X,\rho)$, 
the {\em density constant} $\mu = \mu(\X,\rho)$ 
\citep{gottlieb2013proximity} 
is the smallest number such that any $r$-radius ball in
$\mathcal{X}$ contains at most $\mu$ points at mutual interpoint distance at least $r/2$:
%\[
%\mu(X)=min\{\mu\in \mathbb{N}:x\in \mathcal{X} \rightarrow \mathcal{M}(\frac{r}{2},B_r(x) )\leq \mu\}
%\]
The {\em density dimension} of $\mathcal{X}$ is 
$\dens(\X,\rho) = \log_2\mu(\X,\rho)$.

A dimension property is called {\em hereditary} if it applies to all subspaces of the 
space of interest, that is if 
$\prop(\X') \le \prop(\X)$ for all $\X' \subset \X$.
The doubling constant is known to to {\em semi-hereditary} in that
$\lambda(\X') \le \lambda(\X)^2$ for all $\X' \subset \X$.
The mild increase in the doubling constant is due to the fact that points that serve
as the centers of covering balls of $\X$ may not be present in $\X$, 
necessitating the use of other centers which may not cover all points.

We note that the stated definition of balls 
-- and therefore, of the doubling and density constants --
assumed a symmetric distance function and therefore
is ill-posed for the asymmetric distances of a quasi-metrics.
Addressing this issue is a central component of this paper, 
see Section \ref{sec:dim}.

%An $r-net$ of a set $A \subseteq \mathcal{X}$ is any maximal subset $A$ having mutual interpoint distance at
%least $r$. The r-packing number $M(r, A)$ of $A$ is the maximum size of any $r-net$ of $A$:
%\[
%\mathcal{M}(r, A) = max \{|E| : E \subseteq A,(x, y \in E)  \Rightarrow \rho(x, y) \leq r\}.
%\]

\paragraph{Samples and compression.} 
In a slight abuse of notation, we will blur the distinction
between $S\subset \X$ as a collection of points in a quasimetric space and  
$S\in \X \times \{-1,1\})^n$ as a sequence
of labeled examples. Thus, the notion of a {\em sub-sample}
$\tilde{S}\subset S$ partitioned into its positively and negatively
labeled subsets as $\tilde{S} = \tilde{S}_+ \cup \tilde{S}_-$ is well-defined. 

In metric and semi-metric spaces, one can condense the sample $S$ to a 
{\em consistent} subset $S' \subset$ of sizes
$\left( \frac{\diam}{\rho(S_+,S_-)} \right)^{O(\log \lambda(S,\rho))}$
and 
$\left( \frac{\diam}{\rho(S_+,S_-)} \right)^{O(\log \mu(S,\rho))}$,
respectively. This means that for any
$x \in S_+$ the nearest neighbor of $x$ in $S'$ is some positively labelled point,
while for any
$x \in S_-$ the nearest neighbor of $x$ in $S'$ is some negatively labelled point.
It follows that a nearest-neighbor classifier using the condensed set $S'$
correctly classifies all points of $S$.

Strong generalization bounds are known for classifiers via sample compression.
For consistent classiers, we have:

\begin{theorem}[\citet{GraepelHS05}]
\label{thm:gen-slow}
For any distribution over $\X\times\set{-1,1}$,
any $n\in\N$ and any $0<\delta<1$,
with probability at least $1-\delta$ over the random sample
$S$ of size $n$, the following holds:
If hypothesis $h_S$ queries only a $k$-point subset $S' \subset S$ 
(that is, $h_{S'}(x) = h_S(x)$ for all $x \in \X$) 
then
${\ds
\err(h_S) \le \oo{n-k}\paren{
(k+1)\log n
+\log\oo\delta}.
}$
\end{theorem}

For classifiers with sample error, we have:

\begin{theorem}[\citet{gottlieb2017nearly}]
\label{thm:gen-fast}
Fix a distribution over $\X\times\set{-1,1}$,
an $n\in\N$ and
$0<\delta<1$.
With probability at least $1-\delta$ over the random sample
$S
$ of size $n$, the following holds for all $0\le\eps\le\oo2$:
If hypothesis $h_S$ queries only a $k$-point subset $S' \subset S$ 
(that is, $h_{S'}(x) = h_S(x)$ for all $x \in \X$) 
and misclassifies only an $\eps$ fraction of points in $S$, then
putting $\tips = {\eps n}/({n-d})$, we have
\begin{eqnarray}
\err(h_S) 
&\le&
\tips
+
\frac{2}{3(n-k)}\log\frac{n^{k+1}}\delta
+
\sqrt{ \frac{9\tips(1-\tips)}{2(n-k)}\log\frac{n^{k+1}}\delta}
\label{eq:qdef}
%\\
%&=:& 
%Q(d,\eps)
%.
\end{eqnarray}
\end{theorem}

\section{Dimension of quasi-metric spaces and learning}\label{sec:dim}

In metric spaces, the doubling dimensional is known to control the quality of 
learning via sample compression. We wish to apply the same approach
for quasi-metrics, but here the doubling dimension is not well defined, since
the classic definition of a ball assumes a symmetric space. 
This motivates us to define analogous notions of balls and dimensions in
quasi-metrics, and apply them to learning.

\subsection{Directional covering}\label{sec:cover}

\begin{definition}
For a quasi-metric $(\X,\rho)$, 
define
$B^{\out}_r(x)=\{y \in \X: \rho(x,y)\leq r\}$
and 
$B^{\inn}_r(x)=\{y \in \X: \rho(y,x)\leq r\}$.
\end{definition}

These two distinct notions of balls give rise to two distinct notions of 
covering constants:

\begin{definition}
For a quasi-metric $(\X,\rho)$,
let its {\em outer-constant} 
$\lambda^{\out} = \lambda^{\out}(\X,\rho)$
be the smallest value such that for every radius $r$ and center-point $x \in \mathcal{X}$,
$B^{\out}_r(x)$ can be covered by $\lambda^{\out}$ balls of the form
$B^{\out}_{r/2}(y)$ (where $y \in \X$).

Likewise, let the {\em inner-constant}
$\lambda^{\inn} = \lambda^{\inn}(\X,\rho)$
be the smallest value such that for every radius $r$ and center-point $x \in \mathcal{X}$,
$B^{\inn}_r(x)$ can be covered by $\lambda^{\inn}$ balls of the form
$B^{\inn}_{r/2}(y)$ (where $y \in \X$).
\end{definition}

The definitions of outer-constant and inner-constant are closely related,
and $\lambda^{\out}, \lambda^{\inn}$ can be interchanged by 
simply reflecting the distance funtion, that is swapping the values
$\rho(x,y)$ and $\rho(y,x)$ for all $x,y \in \X$.
Nevertheless, the value of the outer-constant and inner-constant of 
a single quasi-metric may be vastly different:

\begin{lemma}
There exists a quasi-metric $(\X,\rho)$ for which
$\lambda^{\out}=O(1)$
while
$\lambda^{\inn}=n$,
and vice-versa.
\end{lemma}

\begin{proof}
Consider a directed graph $G=(V,E)$ with $n$ vertices $v_1,\ldots,v_{n} \in V(G)$,
where $E$ contains directed edges of length 1 connecting all pairs $v_i,v_{i+1}$ ($1 \le i < n$),
and directed edges of length $1$ connecting $v_i$ to $v_1$ for all $1 < i \le n$.
(This graph is illustrated in Figure \ref{fig:directed graph}.)

Consider any ball of the form $B^{\out}_r(v_i)$ (where $r \ge 1$);
this ball contains the points $v_j$ for the two (possibly overlapping) ranges 
$j \in [i,\min \{i+r, n \}]$
and 
$j \in [1,\min \{r-1, n \}]$.
The three points 
$v_i, v_{i+\lceil r/2 \rceil}, v_{\lceil r/2 \rceil -1}$
cover the two ranges.
Now consider the ball $B^{\inn}_1(v_1)$ -- this is a ball of radius 1 covering all points.
Clearly, $n$ balls of the form $B^{\inn}_{1/2}(v_i)$ are required to cover the entire space.

The reverse claims follows trivially by reversing the direction of the edges.
\end{proof}

\begin{figure}[ht]
    \centering
    \includegraphics[scale=0.40]{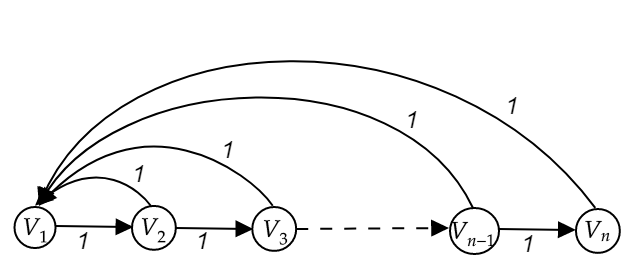}
    \caption{A directed graph with low inner-constant and high outer-constant}
    \label{fig:directed graph}
\end{figure}

Having defined the outer- and inner-constants, we can show that each one
can be used to bound the size of a set covering the space (Lemma \ref{lem:cover}), 
and by extension that learning is possible in quasi-metrics with bounded outer- or inner-constants
(Theorem \ref{thm:gen}).
This is parallel to the doubling constant controlling compression in metric spaces,
and the density constant controlling compression in semi-metric spaces.

As usual, define the {\em diameter} of quasi-metric $(\X,\rho)$ to be 
$\diam = \diam(\X,\rho) = \max_{x,y \in \X} \rho(x,y)$.
A subset $C \subset \X$ is called an $\alpha$-{\em outer-cover} for $\X$
if for all $x \in \X$ we have $\rho(C,x) \le \alpha$.
Likewise, a subset $C \subset \X$ is an $\alpha$-{\em inner-cover} for $\X$
if for all $x \in \X$ we have $\rho(x,C) \le \alpha$.
We can show the following:

\begin{lemma}\label{lem:cover}
Let $(\X,\rho)$ be a quasi-metric of diameter 
$\diam = \diam(\X,\rho)$.
Then $\X$ admits an outer-cover of size at most
$(\lambda^{\out})^{\lceil \log (\diam / \alpha) \rceil}$,
and an inner-cover of size at most
$(\lambda^{\inn})^{\lceil \log (\diam / \alpha) \rceil}$.
\end{lemma}

\begin{proof} 
We prove the outer-cover claim, and proof of the inner-cover claim is similar:
$\X$ can be covered by $\lambda^{\out}$ balls of the type 
$B^{\out}_{\diam/2}(x)$.
Assign each point of $\X$ to its covering ball (or to one of its covering balls
if it is covered by multiple balls.)
Then each of these $\frac{\diam}{2}$-radius balls can be covered by
$\lambda^{\out}$ balls of the type 
$B^{\out}_{\diam/4}(x)$.
Continue this procedure recursively for a total of 
$\lceil \log (\diam / \alpha) \rceil$
steps until reaching balls of diameter at most $\alpha$.
The centers of all balls of this radius constitute an
$\alpha$-outer-cover with the claimed size.
\end{proof}

In the next section, we show that a small outer- or inner-cover can
used for learning.

\subsection{Learning via compression}\label{sec:gen}

Given a sample $S = S_+ \cup S_-$ 
and distance function $\rho$ 
such that $(S,\rho)$ is a quasi-metric,
we will utilize the outer- or inner- constant to produce 
a consistent classifier 
(that is a classifier with no sample error on $S$)
and prove generalization bounds for it.

Consider the margin from all positive points to all
negative points, $\rho^{\pm} = \rho(S_+,S_-)$. 
If we extract 
from $S_+$ a $\rho^{\pm}$-outer cover 
$C^{\pm}_{\out} \subset S_+$
of size 
$(\lambda^{\out}(S_+,\rho))^{\lceil \log (\diam/\rho^{\pm}) \rceil}$, 
then 
$\rho(C,x) \le \rho^{\pm}$ for all $x \in S_+$,
while 
$\rho(C,x) > \rho^{\pm}$ for all $x \in S_-$.
So $C^{\pm}_{\out}$ 
can be used in a consistent classifier for $S$.
Similarly, we may extract from 
from $S_-$ a $\rho^{\pm}$-inner cover 
$C^{\pm}_{\inn} \subset S_-$
of size 
$(\lambda^{\inn}(S_-,\rho))^{\lceil \log (\diam/\rho^{\pm}) \rceil}$, 
and then 
$\rho(x,C) > \rho^{\pm}$ for all $x \in S_+$,
while 
$\rho(x,C) \le \rho^{\pm}$ for all $x \in S_-$.
So $C^{\pm}_{\inn}$ 
can also be used in a consistent classifier for $S$.

We may also consider the margin from all 
negative points to all positive points, 
$\rho^{\mp} = \rho(S_-,S_+)$,
and as above both an inner-cover of $S_+$ of size
$(\lambda^{\inn}(S_+,\rho))^{\lceil \log (\diam/\rho^{\mp}) \rceil}$
or an outer-cover of $S_-$ of size 
$(\lambda^{\out}(S_-,\rho))^{\lceil \log (\diam/\rho^{\mp}) \rceil}$
can be used to produce a consistent classifier. 
See Figure \ref{fig:cover-margin}.

\begin{figure}[ht]
    \centering
    \includegraphics[scale=0.40]{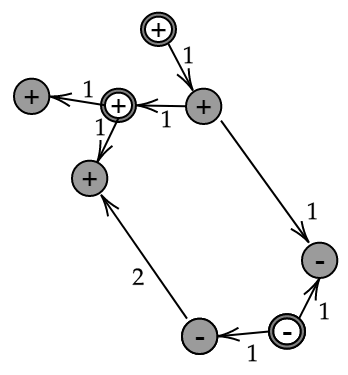}
    \caption{Outer covers and margins for $S_+,S_-$. Here, $\rho^{\pm}=1$ and $\rho^{\mp}=2$.}
    \label{fig:cover-margin}
\end{figure}

Theorem \ref{thm:gen-slow} implies that the size of cover controls 
the generalization bounds of its associated classifier, 
and so of these four possible classifiers, we choose the cover
of the smallest size. We conclude:

\begin{theorem}\label{thm:gen}
For any $(\X,\rho)$ forming a quasi-metric,
any distribution over $\X\times\set{-1,1}$,
any $n\in\N$ and any $0<\delta<1$,
with probability at least $1-\delta$ over the random sample
$S=S_+ \cup S_-$ of size $n$, the following holds:
\begin{eqnarray}
{\ds
\err(h_S) \le \oo{n-k}\paren{
(k+1)\log n
+\log\oo\delta}.
}
\label{eq:gen}
\end{eqnarray}
where
\begin{eqnarray*}
k &=& \min \{ 
(\lambda^{\out}(S_+,\rho))^{\lceil \log (\diam/\rho^{\pm}) \rceil},
(\lambda^{\inn}(S_-,\rho))^{\lceil \log (\diam/\rho^{\pm}) \rceil},	\\
& & 
(\lambda^{\inn}(S_+,\rho))^{\lceil \log (\diam/\rho^{\mp}) \rceil},
(\lambda^{\out}(S_-,\rho))^{\lceil \log (\diam/\rho^{\mp}) \rceil}
\}
\end{eqnarray*}
\end{theorem}

Lemma \ref{lem:cover} and Theorem \ref{thm:gen} 
show how to learn quasi-metrics, but they do not touch upon
the computational complexity and runtime associated with computing 
a cover or a classifier, nor of evaluating the classifier on a new point.
These will be addressed in Section \ref{sec:alg}.

\section{Computational complexity and algorithms}\label{sec:alg}

Here we address the computational issues arising from an implementation
of the classifiers of Theorem \ref{thm:gen}.
In Section \ref{sec:comp-cover}, we address the problem of finding
small $\alpha$-covers, and 
in Section \ref{sec:nns} we show that evaluating the classifiers of
Theorem \ref{thm:gen} requires $\Theta(n)$ distance computations.

\subsection{Computing a cover}\label{sec:comp-cover}

Lemma \ref{lem:cover} demonstrates that a space with small outer- or inner-constant
admits a small outer- or inner-cover. But the proof is non-constructive, and indeed
even in metric spaces finding an optimal cover is NP-hard, and also hard to
approximate within some polynomial factor \citep{gottlieb2014near}. In this section, 
we give three algorithms for producing outer- or inner-covers.

\paragraph{Greedy cover.}
One possible approach to constructing an $\alpha$-cover (whether outer or inner) 
is the {\em arbitrary} algorithm: Choose a point $x \in S$ arbitrarily, add $x$ to the cover $C$, 
remove from $S$ all points $\alpha$-covered by $x$, and repeat. While this algorithm is 
close to the best possible for metric spaces
(for sub-exponential time algorithms \citep{gottlieb2014near}, 
we can show it is arbitrarily bad in quasi-metrics: 
Consider for example a 1-inner-cover for the directed line of Figure \ref{fig:line}.
The arbitrary algorithm may choose the first vertex ($v_1$ in the figure) 
-- which inner-covers no other points -- then the second and third, etc.,
until all points are placed in the cover.

\begin{figure}[ht]
    \centering
    \includegraphics[scale=0.40]{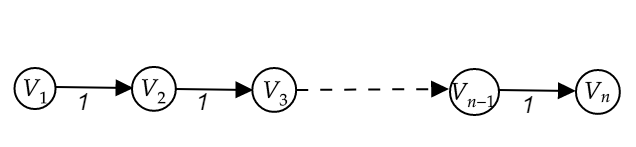}
    \caption{A directed line graph}
    \label{fig:line}
\end{figure}

However, we can show that a simple greedy algorithm gives a $\ln n$-approximation 
to the minimum cover. This algorithm simply chooses the point of $S$ that 
$\alpha$-covers the largest number of other points of $S$, removes all these points
from $S$ and adds the covering point to $C$, and repeats until $S$ is
empty. The greedy construction of an $\alpha$-inner-cover is given in 
Algorithm \ref{alg:coverin}, and construction of an $\alpha$-outer-cover is
similar.

\begin{figure}[ht]
\begin{algorithm}[H]
\SetAlgoLined
\KwData{Sample $S$, parameter $\alpha$.}
\KwResult{$C \subset$ is an $\alpha$-inner-cover for $S$.}
$A \gets S$\;
\While{$A\neq\emptyset$}{
$x=\arg\max_{y \in S} \{|B^{\inn}(y,\alpha) \cap A| \}$\;
$C\gets C\cup \{x\}$\;
$A\gets A \backslash B^{\inn}(x,\alpha)$\;
}
\Return{$C$}
\caption{Greedy inner-cover construction}\label{alg:coverin}
\end{algorithm}
\end{figure}

\begin{lemma}\label{lem:greedy}
Algorithm \ref{alg:coverin} returns an $\alpha$-cover $C$ with cardinality
at most a $(\lceil \ln |S| \rceil + 1)$-factor times the optimal cover.
It can be implemented for quasi-metrics in time 
$O(n^2)$.
\end{lemma}

\begin{proof}
Let $p$ be the size of the optimal cover. This implies that for
any subset $S' \subset S$, there is a point of $S$ that covers
at least $\frac{|S'|}{p}$ points of $S'$.
Then after $p \ln |S|$ iterations, the number of remaining 
points in $A$ is at most
$|S| \left( 1 - \frac{1}{p} \right)^{p (\lceil \ln |S| \rceil + 1)}
\le |S| e^{-\frac{1}{p} \cdot p (\lceil \ln |S| \rceil + 1)}
< 1$,
so all points are covered.

For the runtime, we initially compute for every point the set
of points it covers, and then sort the points into buckets 
depending on the number of points they cover, in total time
$O(n^2 + n \log n) = O(n^2)$.
When a point is removed $A$, all points covering it must be updated
and moved to the adjacent smaller bucket, a cost of $O(n)$
per removed point, for a total of $O(n^2)$.
\end{proof}

It follows that a cover returned by Algorithm \ref{alg:coverin}
can be used to create a classifier satisfying the bounds
of Theorem \ref{thm:gen}, with the dependence on $k$ 
in Equation \ref{eq:gen} replaced by a similar dependence 
on $k (\lceil \ln |S| \rceil + 1)$.

\paragraph{Improved approximation.}
The greedy algorithm gives an additional $\ln n$
factor in the size of the cover -- 
that is total size at most 
$(\lambda^{\inn})^{\log(\diam/\alpha)} \ln n$ --
but this approximation factor may be undesirable.
We can show that a better approximation factor 
can be attained by iteratively executing
the greedy algorithm multiple times. 
Let $\alpha_i$ satisfy
$(\lambda^{\inn})^{\lceil \log(\diam/\alpha_i) \rceil} = \log^{(i)}n$.
Running the greedy algorithm with $\alpha_1$
produces an $\alpha_1$-inner-cover $C_1 \subset S$ of size $O(\log^2 n)$
(where for simplicity we have taken $\lambda^{\inn}$ to be constant
with respect to $n$).
We then run the greedy algorithm to find an 
$\alpha_2$-inner-cover $C_2 \subset S$ for $C_1$, of size 
$O(\log \log n \cdot \log C_1) = O(\log^2 \log n)$.
Repeating this operation until reaching $j$ for which
$\alpha_j \ge \frac{\alpha}{3}$ -- 
that is, fewer than $\log^*n$ times --
we eliminate the dependence on $n$, and replace
it with a factor polynomial in the optimal 
$\alpha$-cover. 
It is easily verify that the set 
$C_{j-1}$ is of size at most
$\exp \left( (\lambda^{\inn})^{O(\log(\diam/\alpha))} \right)$.
See Algorithm \ref{alg:coverratio} for a full description.
From the above analysis we conclude:

\begin{theorem}
Algorithm \ref{alg:coverratio} returns an $\alpha$-inner-cover of cardinality 
$
(\lambda^{\inn})^{O(\log(\diam/\alpha))}
$,
or an $\alpha$-outer-cover of cardinality
$
(\lambda^{\out})^{O(\log(\diam/\alpha))}
$.
It can be implemented for quasi-metrics in time 
$O(n^2 \log^* n)$.
\end{theorem}

\begin{figure}[ht]
\begin{algorithm}[H]
\SetAlgoLined
\KwData{Sample S, margin $\alpha$.}
\KwResult{$C$ is an $\alpha$-inner-net for $S$. }
$i \gets 1$\;
$C \gets S$\;
\While{$\alpha_i < \alpha/3$}{
$C \gets$ Algorithm \ref{alg:subroutine}$(S,C,\alpha_i)$\;
$i \gets i+1 $\;
}
$C \gets$ Algorithm \ref{alg:subroutine}$(S,C,\alpha - \sum_{j=1}^{i-1} \alpha_j)$\;
\Return{$C$}
\caption{Improved $\alpha$-inner-cover construction ($\alpha_i$ is defined in the text)}\label{alg:coverratio}
\end{algorithm}
\end{figure}

\begin{figure}[ht]
\begin{algorithm}[H]
\SetAlgoLined
\KwData{Sample $S$, subset $S'$, parameter $\alpha$.}
\KwResult{$C \subset S$ is an $\alpha$-inner-cover for $S'$.}
$A \gets S'$\;
\While{$A\neq\emptyset$}{
$x=\arg\max_{y \in S} \{|B^{\inn}(y,\alpha) \cap A| \}$\;
$C\gets C\cup \{x\}$\;
$A\gets A \backslash B^{\inn}(x,\alpha)$\;
}
\Return{$C$}
\caption{Greedy cover subroutine}\label{alg:subroutine}
\end{algorithm}
\end{figure}

A cover returned by Algorithm \ref{alg:coverratio}
can be used to create a classifier satisfying the bounds
of Theorem \ref{thm:gen}, with the dependence on $k$ 
in Equation \ref{eq:gen} replaced by a similar dependence 
on $k^{O(1)}$.

\paragraph{Inconsistent cover.}
The previous $\alpha$-cover algorithms required consistency,
meaning that every point in $S_+$ or $S_-$ be $\alpha$-covered.
However, Theorem \ref{thm:gen-fast} gives generalization bounds 
in the presence of errors. That is, even if a computed cover
covers only a $(1-\eps)$ fraction of the points, the bounds
of Theorem \ref{thm:gen-fast} hold with parameters $\eps$ and
\begin{eqnarray*}
k &=& \min \{ 
(\lambda^{\out}(S_+,\rho))^{\lceil \log (\diam/\rho^{\pm}) \rceil},
(\lambda^{\inn}(S_-,\rho))^{\lceil \log (\diam/\rho^{\pm}) \rceil},	\\
& & 
(\lambda^{\inn}(S_+,\rho))^{\lceil \log (\diam/\rho^{\mp}) \rceil},
(\lambda^{\out}(S_-,\rho))^{\lceil \log (\diam/\rho^{\mp}) \rceil}
\}
\end{eqnarray*}
To this end, we modify Algorithm \ref{alg:coverin} to take an 
additional parameter $\eps$, and to terminate when the working
set is sufficiently small: In particular, we replace the condition
`while $A \ne \emptyset$ do' with 
`while $|A| > \eps|S|$ do'.
This gives us the following lemma:

\begin{lemma}
The modified greedy algorithm returns an $\alpha$-cover
with cardinality at most a 
$\lceil \ln(1/\eps) \rceil$-factor
times optimal.
The returned $\alpha$-cover covers at least a 
$(1-\eps)$-fraction of the points.
It can be implemented for quasi-metrics in time 
$O(n^2)$.
\end{lemma}

\begin{proof}
As in the proof of Lemma \ref{lem:greedy},
let $p$ be the size of the optimal cover. This implies that for
any subset $S' \subset S$, there is a point of $S$ that covers
at least $\frac{|S'|}{p}$ points of $S'$.
Then after $p \lceil \ln (1/\eps) \rceil$ iterations, 
the number of remaining points in $A$ is at most
$|S| \left( 1 - \frac{1}{p} \right)^{p (\lceil \ln (1/\eps) \rceil)}
\le |S| e^{-\frac{1}{p} \cdot p (\lceil \ln (1/\eps) \rceil)}
\le \eps |S|$.

The runtime of the modified greedy algorithm is the same
as for the original greedy algorithm.
\end{proof}

A cover returned by the modified greedy algorithm 
can be used to create a classifier satisfying the bounds
of Theorem \ref{thm:gen-fast}, with the dependence on $k$ 
in Equation \ref{eq:qdef} replaced by a similar dependence 
on $k \lceil \ln (1/\eps) \rceil$, where $k$ is as above.

\subsection{Nearest neighbor search}\label{sec:nns}

The classifier of Theorem \ref{thm:gen}
requires the evaluation of the distance of a query point to or from 
$S_+$ or $S_-$,
which reduces to nearest neighbor search.
Note that in doubling spaces, there exist 
$(1+\eps)$-approximate nearest neighbor search algorithms with fast run-time 
$\lambda^{O(1)}\log n + \lambda^{O(\log (1/\eps))}$ \citep{KL04,HM06,CG06},
Thus, instead of constructing a classifier based on an
$\alpha$-cover and then executing an exact nearest neighbor search to
$S_+$ or $S_-$,
one can instead construct a classifier based on a
$\frac{\alpha}{2}$-cover, and execute a fast $2$-approximate nearest neighbor
search, which will correctly classify the query point.
However, we can show the situation for nearest neighbor search for quasi-metrics 
is significantly worse than for metrics:

\begin{lemma}
Let $(\X,\rho)$ be a quasi-metric.
There exists a subset $S$ for which an
$(1+\eps)$-approximate nearest neighbor search 
minimizing $\rho(q,S)$ (respectively,  $\rho(S,q)$)
for $S$ and some $q \in \X$ may require 
$\theta(n)$ distance computations,
even when 
$\lambda^{\inn}(S,\rho)$
and
$\lambda^{\inn}(S \cup q,\rho)$
are constant
(respectively,
$\lambda^{\out}(S,\rho)$
and
$\lambda^{\out}(S \cup q,\rho)$
are constant).
\end{lemma}

\begin{proof}
We prove the case of $\rho(q,S)$, and the 
case of $\rho(S,a)$ is similar.
Consider the case where $S$ is a full binary $2$-HST with edges directed towards the root.
The tree has depth $p$, and an edge connecting a node to its depth $i$ parent has length $2^{-i}$.
The query possesses an infinitesimally small edge directed to a single leaf, 
no edges to any other leaf, and edges of length 
$\sum_{i=j}^{p-1} 2^{-i}$ 
to all nodes of depth $j<p$. 
It is easily verified that both
$\lambda^{\inn}(S,\rho)$
and 
$\lambda^{\inn}(S \cup q,\rho)$
are constant.

As the distance from $q$ to any $j$-level ($j<p$) point is the same,
$q$ must be compared to all leaves to find its nearest neighbor, 
at a cost of $\Theta(n)$ comparisons.
\end{proof}

It follows that $\Theta(n)$ comparisons may be necessary to classify
a query point, and so there does not exist a search algorithm asymptotically
better than brute-force search.

\section{Learning by transformations into metric and semi-metric spaces}\label{sec:trans}

Previously, we showed that learning is possible when either 
the inner- or outer-constant of the sample is small.
However, there is a shortcoming in that these properties are
not hereditary or even semi-hereditary:
Take for example the distance function defined on
a full binary 2-HST, with each edge directed towards the parent.
The quasi-metric implied by this 2-HST has constant 
inner-constant, but the subset including only the root and leaves form a spoke graph,
which has inner-constant $\Theta(n)$
(see Figure \ref{fig:hst2}). 
Nevertheless, this weak notion is sufficient to enable learning whenever
the sample has low inner- or outer-constant.
We also note that a good sample can be guaranteed if we make some very
mild assumptions on the weight distribution of covering sets.

\begin{figure}[ht]
    \centering
    \includegraphics[scale=0.40]{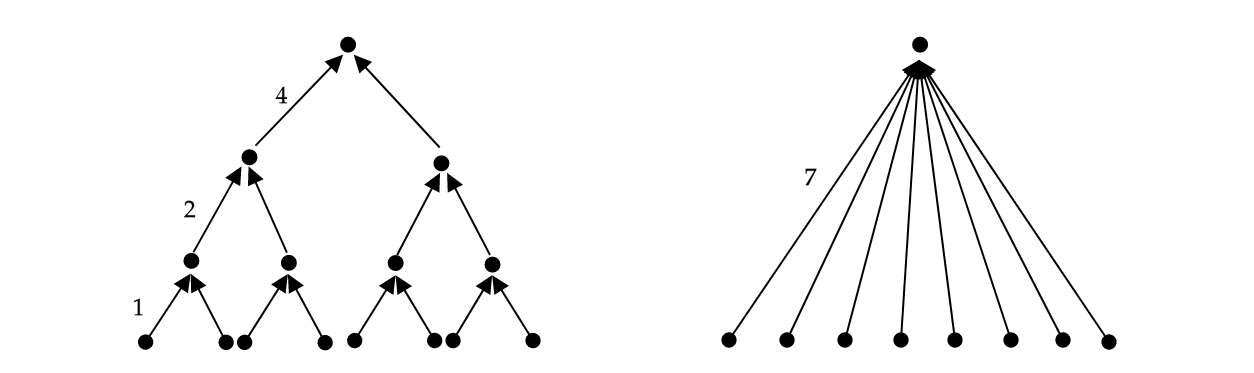}
    \caption{A 2-HST with low inner-constant and subgraph with high inner-constant.}
    \label{fig:hst2}
\end{figure}

The above shortcoming motivates us to consider metric spaces
(for which the doubling dimension is semi-hereditary)
and semi-metric spaces (for which the packing dimension is
hereditary). We ask whether there exists simple transformations
from quasi-metric to metric or semi-metric spaces, and whether
these transformations can be used in learning.
To this end, define 
$\rho^{\max}(x,y) = \max \{\rho(x,y), \rho(y,x)\}$
and 
$\rho^{\min}(x,y) = \min \{\rho(x,y), \rho(y,x)\}$.
We can show the following concerning the $\rho^{\max}$ 
distance function:

\begin{theorem}\label{thm:max}
If $(\X,\rho)$ is a quasi-metric, then 
\begin{enumerate}
\item
$(\X,\rho^{\max})$
is a metric.
\item
$\lambda(\X,\rho^{\max})$ 
may be equal to $n$, even if both
$\lambda^{\out}(\X,\rho)$ and $\lambda^{\inn}(\X,\rho)$ 
are constant.
\end{enumerate}
\end{theorem}

\begin{figure}[ht]
    \centering
    \includegraphics[scale=0.40]{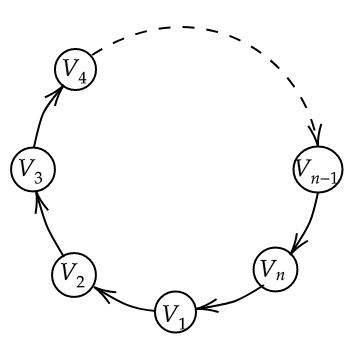}
    \caption{Directed cycle graph}
    \label{fig:circle}
\end{figure}

\begin{proof} 
For the first item:
As $\rho$ satisies the triangle inequality, we have 
$\rho(x,y) \leq \rho(x,z) + \rho(z,y)$ 
and 
$\rho(y,x) \leq \rho(y,z) + \rho(z,x)$
for all $x,y,z \in \X$. 
It follows that
\begin{eqnarray*}
\rho^{\max}(x,y)
&=& 	\max \{ \rho(x,y), \rho(y,x) \}	\\
&\leq& 	\max \{ \rho(x,z) + \rho(z,y) , \rho(y,z) + \rho(z,x) \}	\\
&\leq&  \max \{ \rho(x,z),\rho(z,x) \} + \max \{ \rho(y,z), \rho(z,y) \} \\
&=&	\rho^{\max}(x,z) + \rho^{\max}(z,y).
\end{eqnarray*} 

For the second item:
Consider the cycle graph of Figure \ref{fig:circle}.
Clearly, the graph has outer- and inner-constant 2.
When we consider the distance function $\rho^{\max}$
operating on this graph, we have that all inter-point 
distances are in the range $[\frac{n}{2},n-1]$.
So a ball of radius $n-1$ rooted an any point covers
all points, but a covering of these points by balls of 
radius $\frac{n-1}{2}$ is of size $n$.
\end{proof}

It follows that quasi-metrics can easily be transformed into metrics,
but at the cost of losing the entire structure that permits learning.
Even if the original quasi-metric had both low 
outer- and inner- constants, the resulting metric may have high doubling
constant for which no compression and learning guarantees are possible.\footnote{%
As an aside, we note that the sum operator 
$\rho^{+}(x,y) = \rho(x,y)+\rho(y,x)$
has properties similar to $\rho^{\max}$,
in that it produces a metric with potentially large doubling dimension.
The proof is similar to that of Theorem \ref{thm:max}.
}

Moving to semi-metrics, we can show the following concerning
the $\rho^{\min}$ distance function:

\begin{theorem}
If $(\X,\rho)$ is a quasi-metric, then
\begin{enumerate}
\item
$(\X,\rho^{\min})$
is a semi-metric, but may not obey the triangle inequality.
\item
$\lambda(\X,\rho^{\min}) \le \lambda^{\out}(\X,\rho) + \lambda^{\inn}(\X,\rho)$.
\item
$\mu(\X,\rho^{\min}) \le (\lambda^{\out}(\X,\rho))^2 + (\lambda^{\inn}(\X,\rho))^2$.
\end{enumerate}
\end{theorem}

\begin{figure}[ht]
    \centering
    \includegraphics[scale=0.40]{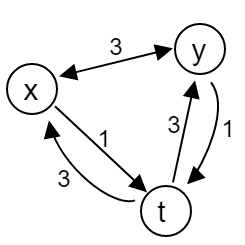}
    \caption{Transformation from quasi-metric to Semi-metric}
    \label{fig:semi-metric}
\end{figure}

\begin{proof}
For the first item:
$\rho^{\min}(x,y) = \min \{ \rho(x,y),\rho(y,x) \} \ge 0$,
and further
$\rho^{\min}(x,y) 
= \min \{ \rho(x,y),\rho(y,x) \} 
= \min \{ \rho(y,x),\rho(x,y) \}
= \rho^{\min}(y,x)$,
so the distance function is non-negative and symmetric.
To show that it may violate the triangle inequality,
refer to Figure \ref{fig:semi-metric}: 
It is easily verified that the graph satisfies the triangle
inequality, however we have
$\rho^{\min}(x,y) = 3 > 1+1 = \rho^{\min}(x,z) + \rho^{\min}(z,y)$,
so the triangle inequality does not hold under $\rho^{\min}$.

For the second item: 
Take any point $x \in X$ and radius $r$,
and let $B$ be the points in $B_r(x)$ under distance measure $\rho^{\min}$.
Let $B^{\inn} \subset B_r(X)$ include all points $y \in B_r(x)$ 
satisfying $\rho(y,x) = \rho^{\min}(y,x)$, and let
$B^{\out} = B_r(X) - B^{\inn}$.
Now, as $\rho^{\min}$ does not expand distance of $\rho$,
the at most $\lambda^{\inn}$ points that served as an $\frac{r}{2}$-inner-cover of $B^{\inn}$
under $\rho$ still $\frac{r}{2}$-covers those points under $\rho^{\min}$.
Likewise, the at most $\lambda^{\out}$ points that served as an $\frac{r}{2}$-outer-cover of $B^{\out}$
under $\rho$ still cover those points under $\rho^{\min}$.
The claim follows.

For the third item:
The proof is similar to the second item, except we look at the 
$\frac{r}{4}$-inner-cover and $\frac{r}{4}$-outer-cover points, 
of which there are (by Lemma \ref{lem:cover}) at most
$(\lambda^{\out}(\X,\rho))^2$ and 
$(\lambda^{\inn}(\X,\rho))^2$
respectively.
All points covered by a single cover point under $\rho^{\min}$
are within distance $\frac{r}{2}$,
and at most one can be a witness for the density constant 
with respect to an $r$-ball.
The claim follows.
\end{proof}

We conclude that if both the inner- and outer-constants are small,
we may learn by using the $\rho^{\min}$ operator to transform the
quasi-metric into a semi-metric. The semi-metric has the useful property
that its learning is controlled by the density constant, which is a
hereditary property. Nevertheless, this transformation comes at a price, 
as the margin (which controls learning together with the density constant)
is now reduced to 
$\min \{ \rho^{\pm}, \rho^{\mp} \}$
(where 
$\rho^{\pm} = \rho(S_+,S_-)$
and 
$\rho^{\mp} =  \rho(S_-,S_+)$).
In contrast, the quasi-metric bounds of Theorem \ref{thm:gen} 
allow us to choose whichever value of 
$\rho^{\pm}$ and $\rho^{\mp}$
yields better bounds.

\bibliography{references}

\end{document}